\documentclass[conference]{IEEEtran}
\usepackage{amsmath,amsfonts,amssymb,latexsym,amsthm,color} 
\usepackage{comment}

\usepackage[utf8]{inputenc} 
\usepackage[T1]{fontenc}    
\usepackage{hyperref}       
\usepackage{url}            
\usepackage{booktabs}       
\usepackage{amsfonts}       
\usepackage{nicefrac}       
\usepackage{microtype}      
\usepackage{graphicx}
\usepackage{algorithm}
\usepackage{algpseudocode}
\usepackage{cite}

\usepackage{graphicx} 
\newcommand*{\Scale}[2][4]{\scalebox{#1}{\ensuremath{#2}}} 

\usepackage{amsmath,amssymb,latexsym,amsthm,color}
\usepackage{bbm}
\usepackage{dsfont}
\newtheorem{defi}{Definition}%
\newtheorem{prop}{Proposition}
\newtheorem{cor}{Corollary}
\newtheorem{obs}{Observation}

\def \1{{\mathds{1}}}

\def \bx{\boldsymbol{x}}

\def \bp{\boldsymbol{p}}
\def \bq{\boldsymbol{q}}

\def \bz{\boldsymbol{z}}

\def \bv{\boldsymbol{v}}

\def\bU{\boldsymbol{U}}
\def \bB{\boldsymbol{B}}
\def \bC{\boldsymbol{C}}

\def \bV{\boldsymbol{V}}

\def \bc{\boldsymbol{c}}
\def \o{{\boldsymbol{\omega}}}
\def \O{{\boldsymbol{\Omega}}}

\def \phat{\widehat{\boldsymbol{p}}}
\def \qhat{\widehat{\boldsymbol{q}}}
\def \Var{\mathcal{V}}
\def \Reg{\mathcal{R}}
\def \ineqnew{\mathop{\Scale[1.1]{\lessgtr}}_{\Scale[0.65]{\omega_{\ell}=1}}^{\Scale[0.65]{\omega_{\ell}=-1}}}

\def\R{{\mathbb R}}

\usepackage{amsmath}

\DeclareMathOperator*{\argmin}{arg\,min}

\title{Geometry of the Minimum Volume Confidence Sets}

\author{
\IEEEauthorblockN{Heguang Lin$^\dagger$}
\IEEEauthorblockA{
hlin324@wisc.edu}
\and
\IEEEauthorblockN{Mengze Li$^\dagger$}
\IEEEauthorblockA{
mli562@wisc.edu}
\and 
\IEEEauthorblockN{Daniel Pimentel-Alarcón}
\IEEEauthorblockA{
pimentelalar@wisc.edu}
\and
\IEEEauthorblockN{Matthew L. Malloy}
\IEEEauthorblockA{
matthew.malloy@wisc.edu}
\thanks{$^\dagger$ Undergraduate Student.  All authors are at the University of Wisconsin, Madison, Wisconsin, United States. }
}

\IEEEoverridecommandlockouts
\begin{document}

\maketitle

\begin{abstract}
Computation of confidence sets is central to data science and machine learning, serving as the workhorse of A/B testing and underpinning the operation and analysis of reinforcement learning algorithms \cite{jamieson2014lil}.  This paper studies the geometry of the \emph{minimum-volume confidence sets} for the multinomial parameter.  When used in place of more standard confidence sets and intervals based on bounds and asymptotic approximation, learning algorithms can exhibit improved sample complexity. Prior work \cite{malloy2021ISIT} showed the minimum-volume confidence sets are the level-sets of a discontinuous function defined by an exact $p$-value.  While the confidence sets are optimal in that they have minimum average volume, computation of membership of a single point in the set is challenging for problems of modest size.  Since the confidence sets are level-sets of discontinuous functions, little is apparent about their geometry.  This paper studies the geometry of the \emph{minimum volume confidence sets} by enumerating and covering the continuous regions of the exact $p$-value function.  This addresses a fundamental question in A/B testing: given two multinomial outcomes, how can one determine if their corresponding minimum volume confidence sets are disjoint?  We answer this question in a restricted setting.   
\end{abstract}

\section{Introduction}
Confidence sets, regions, and intervals are fundamental tools in data science, statistical inference, and machine learning, capturing a range of plausible beliefs of the parameters of a model. For simplicity of computation and analysis, most approaches to construct confidence sets rely on approximation or bounds that are loose in the small sample regime \cite{casella2021statistical, chafai2009confidence, malloy2021ISIT}. While these approaches are often optimal asymptotically, tighter confidence sets in the small sample regime can reduce sample complexity in A/B testing, reinforcement learning algorithms, and other problems in applied data science \cite{malloy2020optimal, jamieson2013finding, malloy2015contamination, malloy2012quickest, malloy2013sample}.


Finding tight confidence sets for categorical distributions is a long studied problem.  The goal is to construct sets of minimal volume (i.e, as small as possible) that contain the true parameter with high confidence.
%
%
%
Recent work \cite{malloy2021ISIT} studied a confidence set construction based on level-sets of the exact $p$-value function that satisfies a minimum volume property. Averaged over the possible empirical outcomes, \cite{malloy2021ISIT} showed that the confidence sets proposed in \cite{chafai2009confidence} have minimum volume among any confidence set construction.  The result is based on a duality between hypothesis testing and confidence sets, and a is specific instance of a general theory of optimal confidence sets \cite{brown1995optimal} first observed in restricted settings in the traditional work of Sterne \cite{sterne1954some} and Crow \cite{crow1956confidence}. 

The minimum volume confidence sets (MVCs) for the multinomial parameter are defined by the level-sets of the exact $p$-value.  To compute membership of a single parameter value in the set, one must compute the exact $p$-value; naively, this involves enumerating and computing partial sums of all the empirical outcomes of $n$ i.i.d.~observations that may take one of $k$ possible values. While this direct approach to computing the $p$-value scales as $n^k$, recent work has \cite{resin2020simple} reduced this computation to $(\sqrt{n})^k$. Nonetheless, checking membership of a single parameter value in the confidence set becomes prohibitive for modest values of $k$ and $n$.  


This computational limitation becomes even more challenging for basic applications.  A common task is to observe two empirical outcomes and determine if the corresponding confidence sets are disjoint.  This arises in A/B testing; if two confidence sets are disjoint, then the underlying multinomial parameters associated with the outcomes are different (to within a significance specified by the confidence-level).  A naive approach is to grid the multinomial parameter values and check for a value that lies in the confidence sets of both outcomes.  Unfortunately, this fails to guarantee an empty intersection, as the MVCs can have irregular geometry, including arbitrarily small disconnected regions. As they are constructed from the level-sets of a discontinuous function, the MVCs do not satisfy properties such as convexity, radial-convexity, or connected-ness (see Fig. \ref{fig:confidence_set}).  This contrasts with traditional confidence sets where an empty intersection can be determined by exploiting geometric properties such as convexity. 

This paper studies the geometry of the {\em minimum volume confidence sets for the multinomial parameter}.  We describe an algorithm that enumerates and covers the regions of the simplex over which the exact $p$-value function is continuous, enabling numerical characterization of the MVCs. As a consequence, we answer a basic question in A/B testing in the restricted setting of three categories.  Given two multinomial outcomes, how can one determine if their corresponding confidence sets are disjoint?  The numerical characterization of the sets, facilitated by the enumeration and covering of the continuous regions of the exact $p$-value function provides a definitive answer to this question.  The approach sheds light on the answer for more than three categories, but this remains an open question.

\section{Notation and Basic Definitions}
Let $X_1, \dots, X_n$ be i.i.d.\ samples of a categorical random variable that takes one of $k$ possible values from a finite number of categories $\mathcal{X} = \{x_1, \dots, x_k\}$. The empirical distribution $\phat$ is the relative proportion of occurrences of each element of $\mathcal{X}$ in $X_1, \dots, X_n$, i.e., $\phat = [\nicefrac{n_1}{n}, \dots, \nicefrac{n_k}{n}]$, where ${n}_i = \sum_{j=1}^n {\1_{\{ X_j = x_i \} }  }$. Let $\Delta_{k,n}$ denote the discrete simplex from $n$ samples over $k$ categories:
\begin{align*}
\Delta_{k,n} \ := \ \left\{ \phat \in \{0, \ \nicefrac{1}{n},\ \nicefrac{2}{n}, \ \dots, \ 1\}^k \ : \  \sum_{i=1}^k \widehat{p}_i =1  \right\},
\end{align*}
and define $m = |\Delta_{k,n}|= { n+k-1 \choose k-1}$. 
Denote the continuous simplex as $\Delta_k = \left\{\bp \in [0,1]^k :  \sum_i p_i =1  \right\}$. We use $\mathcal{P}(\Delta_{k,n})$ to denote the power set of $\Delta_{k,n}$, and $\mathcal{P}(\Delta_k)$ to denote the set of Lebesgue measurable subsets of $\Delta_k$. For any $\mathcal{S} \subset \Delta_{k,n}$ we write $\mathbb{P}_{\boldsymbol{p}}(\mathcal{S})$ as shorthand for  $\mathbb{P}_{\bp} \left(\left\{ X \in \mathcal{X}^n : \phat(X) \in \mathcal{S}\right\} \right)$, where $\mathbb{P}_{\bp}( \cdot)$ denotes the probability measure under the multinomial parameter $\bp \in\Delta_{k}$.

\begin{defi}(Confidence set)  Let $\mathcal{C}_{\alpha}( \phat): \Delta_{k,n} \rightarrow \mathcal{P}(\Delta_{k})$ be a set valued function that maps an observed empirical distribution $\phat$ to a subset of the $k$-simplex. $\mathcal{C}_{\alpha}( \phat)$   is a \emph{confidence set} at confidence level $1-\alpha$ if the following holds:

\begin{eqnarray} \label{eqn:cr}
\sup_{\bp \in \Delta_{k}} \mathbb{P}_{\bp}  \left( \bp  \not \in \mathcal{C}_{\alpha}(\phat) \right) \ \leq \ \alpha.
\end{eqnarray}
\label{def:cr}
\end{defi}




\begin{defi} ($p$-value) Fix an outcome $\phat$.  The 
$p$-value as a function of the null hypothesis $\bp$ is given by:
\begin{align} \label{eqn:partial}
\rho_{\phat}(\bp) \ = \ \sum_{\qhat \in \Delta_{k,n} : \mathbb{P}_{\bp}(\qhat) \leq \mathbb{P}_{\bp}(\phat)  } \mathbb{P}_{\bp}  \left( \qhat  \right).
\end{align}
\end{defi}
\noindent
For a fixed outcome $\widehat{\bp}$, we write $\rho(\bp)$ for simplicity.

\begin{prop}
\label{def:minvol}
(Minimum volume confidence set (MVCs) \cite{malloy2021ISIT}).  
The MVCs are defined as 
\begin{eqnarray*} 
\mathcal{C}_{\alpha}^\star(\phat) \ := \ \big\{\bp \in\Delta_{k} \ : \ \rho_{\phat}(\bp) \geq \alpha \big\},
\end{eqnarray*}
and satisfy 
\begin{eqnarray*}
\sum_{\phat \in \Delta_{k,n} } \mathrm{vol} \left( \mathcal{C}_{\alpha}^\star(\phat) \right) \leq
\sum_{\phat \in \Delta_{k,n} } \mathrm{vol} \left( \mathcal{C}_{\alpha}(\phat) \right)
\end{eqnarray*}
for any confidence set $\mathcal{C}_{\alpha}( \cdot )$; here $\mathrm{vol}( \cdot ) $ denotes the Lesbague measure.  A proof can be found in \cite{malloy2021ISIT}. 
\end{prop}

\section{Geometry of the Minimum Volume \\ Confidence Sets} \label{sec:Geometry}

By definition, the MVCs are the level-sets of the $p$-value function, which is a discontinuous function of $\bp$, because it is a partial sum of the multinomial outcomes. In particular, an arbitrarily small change in $\bp$ can include or exclude new terms of the sum in \eqref{eqn:partial}.  For a region of the simplex over which the terms included in the partial sum do not change, the $p$-value \emph{is} continuous, as it is a sum of continuous functions.  If terms included in the partial sum change, a discontinuity may occur.

\begin{figure}
\centering
\includegraphics[width=0.5\textwidth]{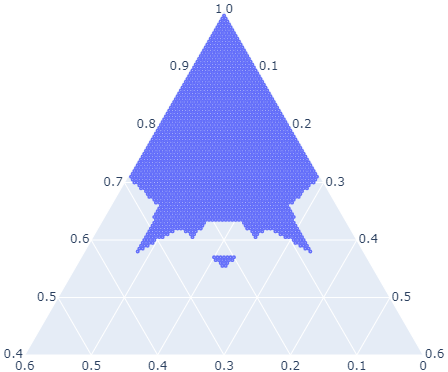}
\caption{An example of disconnected MVCs (blue) with observation $\phat = [0,1,0]$ and confidence level 0.5 for $k=3,n=4$. Note the figure represents a corner of the simplex, as specified by the range on the axis. 
\label{fig:confidence_set}}
\end{figure}

The terms included in the partial sum in (\ref{eqn:partial}) correspond to $\{\qhat \in \Delta_{k,n} : \mathbb{P}_{\bp}(\qhat) \leq \mathbb{P}_{\bp}(\phat)\}$. Consequently, the discontinuities of $\rho_{\phat}(\bp)$ occur whenever $\mathbb{P}_{\bp}(\phat) = \mathbb{P}_{\bp}(\qhat)$ for some $\qhat \in \Delta_{k,n}\backslash \phat$. Observe that $\mathbb{P}_{\bp}(\phat)$ is fully characterized by the multinomial distribution with parameter $\bp$ as:
\begin{align*}
\mathbb{P}_{\bp}(\phat) 
\ = \ \frac{n!}{(n\widehat{p}_1)!  \ldots (n\widehat{p}_k)!} p_1^{n\widehat{p}_1}  \cdots p_k^{n\widehat{p}_k}.
\end{align*}
It follows that the condition $\mathbb{P}_{\bp}(\phat) = \mathbb{P}_{\bp}(\qhat)$ can be rewritten as the following equation:
\begin{align*}
c_0 p_1^{c_1} p_2^{c_2}  \cdots p_k^{c_k} = 1,
\end{align*}
where $c_0:=\frac{(n\widehat{q}_1)!  \cdots (n\widehat{q}_k)!}
{(n\widehat{p}_1)!  \cdots (n\widehat{p}_k)!}$, and $c_i = n(\widehat{p}_i - \widehat{q}_i)$. 
This implies that the discontinuities of $\rho(\bp)$ are characterized by a union of algebraic varieties in the simplex $\Delta_k$.

\begin{defi}(Discontinuity variety) \label{def:dis_var}
Consider a fixed observation $\phat \in \Delta_{k,n}$, and let $\ell=1,\dots,m-1$ enumerate the types $\qhat \in \Delta_{k,n} \backslash \phat$.  We define the {\bf discontinuity variety} $\Var_{\ell}$ as the intersection of the simplex $\Delta_k$ with the $(k-1)$-dimensional algebraic variety characterized by
\begin{align}
\label{polyDef}
f_\ell(\bp) &= 1 - c_0 p_1^{c_1} p_2^{c_2}  \cdots p_k^{c_k}.
\end{align}
Notice that $f_{\ell}(\bp)$ is defined over $\R^k$, whereas $\Var_{\ell}$ is a $(k-2)$-dimensional subset of $\Delta_k$, and both have implicit dependence on $\widehat{\bp}$ and $\widehat{\bq}$ through $c_0, \dots, c_{k}$.
\end{defi}

It follows that the union of the discontinuity varieties
\begin{align*}
\bigcup_{\ell =1}^{m-1} \Var_{\ell}
\end{align*}
characterizes all the discontinuities in the $p$-value function, which in turn partitions the simplex $\Delta_k$ into at most $2^{m-1}$ (possibly disconnected) sets. To see this observe that each discontinuity variety $\Var_{\ell}$ splits $\Delta_k$ in two open sets:
\begin{align*}
\{\bp \in \Delta_k : f_\ell(\bp) < 0\},
\hspace{.25cm} \text{and} \hspace{.25cm}
\{\bp \in \Delta_k : f_\ell(\bp) > 0\}.
\end{align*}
Since $\ell \in \{1, \dots, m-1\}$ (there are $m-1$ elements $\qhat$ in $\Delta_{k,n} \backslash \phat$), the discontinuity varieties will split $\Delta_k$ in at most $2^{m-1}$ {\em candidate sets}, each defined by a combination of directions in the {\em splitting inequalities}
\begin{align}
\label{eqn:cont_region}
\big\{ f_\ell(\bp) \ \lessgtr \ 0 \big\}_{\ell=1}^{m-1}.
\end{align}
By construction, no point in these candidate sets satisfies a discontinuity condition $f_{\ell}(\bp)=0$, which implies that $\rho_{\phat}(\bp)$ is continuous in these regions.  However, many of these candidate sets may be empty (if they result from inconsistent splitting inequalities).  Moreover, each non-empty candidate set consists of a finite number of connected regions (a notion we make precise in the following section).  We  refer to each as a {\em continuity region} and formalize these ideas in the  following.

\begin{defi}
(Candidate set; continuity set; continuity region) Given $\o \in \{-1,1\}^{m-1}$, let
\begin{align*}
\Reg_\o \ = \ \left\{ \bp \in \Delta_k : \bigwedge_{\ell=1}^{m-1} f_{\ell}(\bp) \ineqnew 0 \right\}
\end{align*}
be the {\bf candidate set} associated with the combination of splitting inequalities indexed by $\o$ (here $\omega_{\ell}$ denotes the $\ell^{\rm th}$ entry of $\o$).
We say $\Reg_\o$ is a {\bf continuity set} if $\Reg_\o \neq \emptyset$. Furthermore, each continuity set is the union of a finite number of connected subsets, termed {\bf continuity regions}. 
We say $f_\ell(\bp)$ {\bf touches} $\Reg_{\o}$ if the closure of $\Reg_{\o}$ includes a $\bp$ such that $f_\ell(\bp)=0$.
\end{defi}

An example of the continuity regions associated with an observation $\phat$ is shown in Fig. \ref{fig:continuity_regions}.  The $p$-value function $\rho(\bp)$ is continuous over each region. 

\begin{figure}
\centering
\includegraphics[width=0.46\textwidth]{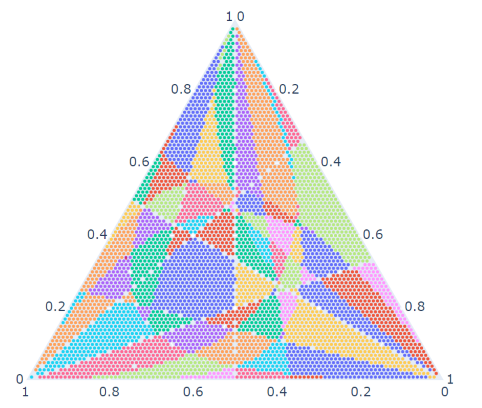}
\caption{Partitioning of the simplex $\Delta_k$ into continuity regions for $n = 4$, $k = 3$, and $\phat = [\nicefrac{1}{4},\nicefrac{1}{2},\nicefrac{1}{4}]$. The $p$-value function is continuous over each continuity set (indicated by the colors). Figure generated by Plotly \cite{plotly}.
\label{fig:continuity_regions}}
\end{figure}

\subsection{Identifying Continuity Sets} \label{sec:finding}
The key to identifying \emph{continuity} sets in the simplex lies in determining whether the each of the $2^{m-1}$ \emph{candidate} sets are empty. To determine if a candidate set $\Reg_\o$ is empty, we first check if its splitting inequalities are feasible.
If they are infeasible, the candidate set is empty. 
On the other hand, if the splitting inequalities result in a non-empty set, we can further check if this set intersects the simplex by finding the minimum and maximum of $p_1+\dots+p_k$, constrained by the splitting inequalities.   As we show, if this minimum value is less one, and the maximum is greater than one, we can conclude that $\Reg_\o$ is a (non-empty) \emph{continuity set}.  We make this approach precise in the following discussion and Alg. \ref{alg:fcr}.

\begin{algorithm} \label{alg:fcr}
\caption{Find \emph{continuity sets}}\label{alg:fcr}
\begin{algorithmic}[1]
\State{\textbf{Input}}: observation $\phat \in \Delta_{k, n}$
\State {initialize: set $\O = \{ \}$ }
\For {$\o \in \{-1, 1\}^{m-1}$}
\State {$\mathcal{S} = \left\{ \bp \in \mathbb{R}^k : \bigwedge_{\ell=1}^{m-1} f_{\ell}(\bp) \ineqnew  0 \right\}$ }
\State{Solve GP: $t_{\mathrm{min}} = \min_{\bp \in \mathcal{S}} \sum_{i=1}^{k}p_i$} 
\If {GP feasible and $ t_{\mathrm{min}} \leq 1$}
\State {$\mathcal{P} =\left\{ \bz \in \mathbb{R}^{k+} : \bigwedge_{\ell=1}^{m-1} \bz^T c_{\ell} \ineqnew  c_\ell \right\}$}
\State {$\{\bv_1, \dots, \bv_j\} \leftarrow$ vertices of $\mathcal{P}$}
\State { $t_{\mathrm{max}} = \max_{\bz \in \{\bv_1, \dots, \bv_j\} }  \sum_{i=1}^{k}e^{-z_i}$}

\If {$t_{\mathrm{max}} \geq 1$} 
    \State {Append $\o$ to $\O$}
\EndIf
\EndIf
\EndFor
\State {{\bf Return}: continuity sets $\O$}
\end{algorithmic}
\end{algorithm}



\label{subsec:GP}
To check feasibility and find the minimum of $p_1+\cdots+p_k$ we employ a Geometric Program (GP). Denote $t_{\mathrm{min}}$ the solution to the following optimization: 
\begin{align*}
t_{\mathrm{min}} = 
    \underset{\bp \in \mathbb{R}^{k+}}{\text{min }} & p_1+\cdots+p_k \\
    \qquad \text{subject to } & 
    \bigwedge_{\ell=1}^{m-1} f_{\ell}(\bp) \ineqnew  0.
\end{align*}
\noindent Note the objective is a posynomial function and the term $c_0 p_1^{c_1} p_2^{c_2}  \cdots p_k^{c_k}$ in the discontinuity variety is a monomial function (which is also a posynomial function). This makes the optimization problem a standard GP \cite{boyd2007tutorial} and we can easily deterimine $t_{\mathrm{min}}$ using off-the-shelf solvers.  

Next, we aim to find the maximum of $p_1+\dots+p_k$, which is most easily accomplished by a logarithmic transformation, which we refer to as \emph{$z$-space}.  While the continuity sets are defined as the intersection of the simplex with a region defined by non-linear splitting equalities, after a logarithmic transformation, the splitting equalities represent hyperplanes.  More specifically, the constraint $f_{\ell}(\bp) = 0$ for $\bp \in \mathbb{R}^{k+}$ is equivalent to 
$ \log(c_{0} p_1^{c_{1}}  \cdots p_k^{c_{k}}) = 0 $
which can be represented as an inner product $\bz^T\bc = c$, a notion we make precise in the following observation.

\begin{obs} \label{obs:zspace} $\bz$-space.
Let $z_i = -\log(p_i)$, $c = \log(c_0)$ and $\bc = [c_1, \dots, c_k]^T$.  Then, 
\begin{eqnarray*}
\left\{ \bp \in \mathbb{R}^{k+}: f_{\ell}(\bp) = 0 \right\}
= \left\{  e^{-\bz} : \bz^T \boldsymbol{c} = c \right\}
\end{eqnarray*}
for $\bz \in \mathbb{R}^{k+}$. 
\end{obs}
\noindent Obs. \ref{obs:zspace} implies the set of splitting inequalities in  
(\ref{eqn:cont_region}) is equivalent linear halfspace inequalities: 
\begin{eqnarray}
\{ \bz^T \bc_\ell \lessgtr c_\ell \}_{\ell=1}^{m-1}
\end{eqnarray}
where $\ell=1,\dots,m-1$ indexes $\widehat{\bq} \in \Delta_{n,k} \setminus \widehat{\bp}$.  The set of splitting equalities is the intersection of linear halfspaces, which forms a polyhedron $\mathcal{P}$.  To compute the maximum of $p_1+\cdots+p_k$ in the original space, we can equivalently maximize $e^{-z_1}+\cdots+e^{-z_k}$ which is convex in $\bz$.  

\begin{obs} \label{obs:conv}
The maximum (if it exists) of a convex function $g(\bz)$ over the polyhedron $\mathcal{P}$ is achieved at one of the vertices $\bv_1,\cdots,\bv_j$ of $\mathcal{P}$, since for $\sum \lambda_i =1$, $\lambda_i \geq 0$,
\begin{eqnarray*}
g\left(\sum_{i} \lambda_i \bv_i\right) \leq \sum_{i} g\left( \lambda_i \bv_i\right) \leq \mathrm{max}_i\left(g(\bv_i) \right)).
\end{eqnarray*}
\end{obs}
\noindent Obs. \ref{obs:conv} allows us to directly compute the maximum by enumerating the vertices of the $\mathcal{P}$.  In particular, define 
\begin{eqnarray}
t_{\max} = \max_{\bz \in \{\bv_1, \dots, \bv_j \} } \sum_{i=1}^{k} e^{-z_i}
\end{eqnarray}
where $\{\bv_1, \dots, \bv_j \}$ are the vertices of $\mathcal{P}$, and set $t_{\max} = \infty$ if the maximum does not exist. This gives rise to the following corollary, which provides conditions under which a candidate set is non-empty. 

\begin{cor}
If $t_\mathrm{max} \geq 1$ and $t_{\min} \leq 1$, then $\mathcal{R}_{\o} \neq \emptyset$.
\begin{proof}
If the splitting inequalities are feasible, we are guaranteed to find the minimum (using the GP) and the maximum (by checking the vertices in $\bz$-space) of $p_1+\dots+p_k$. Since $t_\mathrm{max} \geq 1$ and $t_{\min} \leq 1$, this implies  $p_1 + \dots + p_k = 1$ for some $\bp$ constrained by the splitting inequalities, which follows as the splitting inequalities define a connected subset of $\mathbb{R}^{k+}$.  Together, this implies $\Reg_{\o} \neq \emptyset$.
\end{proof}
\end{cor}

\subsection{Identifying Continuity Set Vertices}
Understanding and enumerating vertices of the continuity sets plays an important role in understanding their geometry of the MVCs. 
\begin{defi} \label{def:vertex}
A \emph{vertex} is a point $\boldsymbol{v} \in \Delta_{k}$ where $k-1$ splitting equalities intersect:
\begin{eqnarray*}
\boldsymbol{v} \in \Delta_k :  f_{\ell_i}(\boldsymbol{v}) = 0,\ \ \ i=1,\dots,k-1.
\end{eqnarray*}
\end{defi}
\noindent We highlight that the vertices defined here are different from the vertices \emph{of a polytope} $\mathcal{P}$ discussed in the previous section.  In particular, Def. \ref{def:vertex} refers to vertices of the continuity sets (which must lie in the simplex) as opposed to the vertices of a polytope $\mathcal{P}$ in $\bz$-space.



\begin{cor} \label{cor:vertex1}
The number of vertices is at most  $2 {m-1 \choose k-1}$.  
\begin{proof}
There are $m-1$ splitting equalities; each choice of $k-1$ results in at most 2 vertices.  To see this note the following two observations.  First, in $\bz$-space, the intersection of $k-1$ splitting equalities is is a 1-d affine space in $\mathbb{R}^k$. 
This follows as the intersection of $k-1$ splitting equalities is the solution set to a system of linear equations $A \bz = \bc$, where $A \in \mathbb{R}^{(k-1)\times k}$ with rows $\bc_{\ell_1}^T,\dots, \bc_{\ell_{k-1}}^T$.  This follows under the assumption that the $\bc_{\ell_1},\dots, \bc_{\ell_{k-1}}$ are in general position and form a linearly independent set.  

Next note that the simplex in $\bz$-space is 
\begin{eqnarray} \label{eqn:z_simplex}
\tilde \Delta_k = \left\{\bz \in \mathbb{R}^{k+} : \sum_{i=1}^k e^{-z_i} =1  \right\}
\end{eqnarray}
and is the exterior of a strictly convex set; specifically, 
$\{ \bz \in \mathbb{R}^{+}: \sum_{i=1}^k e^{-z_i} \leq 1 \}$
is a strictly convex set since $\sum_{i=1}^k e^{-z_i}$ is a strictly convex function for $z_i > 0$.  The level-sets of a strictly convex function are strictly convex sets \cite{boyd2004convex}, implying $\{ \bz \in \mathbb{R}^{+}: \sum_{i=1}^k e^{-z_i} \leq 1 \}$ is a strictly convex set.  

Lastly, a 1-d affine space can intersect the exterior of a strictly-convex set at most twice. \end{proof}
\end{cor}

The vertices can be identified in $\bz$-space by directly enumerating all subsets of $k-1$ splitting inequalities.  For each subset, the (at most two) vertices can be computing numerically using a line search to determine where the $1$-$d$ affine space intersects $\tilde{\Delta}_k$ defined in \ref{eqn:z_simplex}.  Further observations relating continuity regions to vertices can be found in Appendix \ref{app:verts}, Cor. \ref{cor:vertex2} and Cor. \ref{cor:vertex3}.

\subsection{Covering of Continuity Sets}
Fully specifying the MVCS requires specifying the level-sets of the $p$-value function over $\Delta_{k}$.  Numerically, this requires specifying a \emph{discrete covering set} or \emph{cover} for each continuity set.  Note that naive approach to covering a continuity set -- covering the simplex and assigning each point to the continuity set that it belongs -- fails, as it can miss arbitrarily small continuity sets or those with a irregular, `pointy' geometry.

Consider a single continuity set $\Reg_{\o}$.  One approach to covering  $\Reg_{\o}$ is to consider subsets of the discrete simplex comprised of all the points inside $\Reg_{\o}$ \emph{and} all points outside $\Reg_{\o}$ but within a specified distance. 

\begin{defi}
$(\epsilon, \delta)$-cover.  We say $\mathcal{G}_{\o} \subset \Delta_{k,\eta}$ is an $(\epsilon, \delta)$-cover of $\Reg_{\o}$ if for every $\bp \in \Reg_{\o}$ there is a $\bq \in \mathcal{G}_{\o}$ such that $||\bp - \bq||_2 \leq \epsilon$ and for every $\bq \in \mathcal{G}_{\o}$ there is a $\bp \in \Reg_{\o}$ such that $||\bq - \bp||_2 \leq \delta$.
\end{defi}

The cover $\mathcal{G}_{\o}$ requires computing the minimum distance between a point $\bq \in \Delta_{k, \eta}$ and $\Reg_{\o}$.  To facilitate this computation, we proceed in two steps: (1) presenting a necessary condition for such minimum distance points to a single discontinuity variety and enumerating all points that satisfy the necessary condition, and (2) showing that this can be used to find the minimum distance to a continuity set $\Reg_{\omega}$ in the restricted setting of $k=3$.

\begin{cor} \label{cor:orth}
Fix $\bq \in \Delta_k$ and consider a discontinuity variety $f_\ell(\bp)$. Define $\bp^{-1} = [p_1^{-1} \  \dots  \ p_k^{-1}]^T$ and let $\bB \in \mathbb{R}^{k\times k }$ be given as $\bB = \bU\bU^T \mathrm{diag}(c_1, \dots, c_k)$ where the columns of $\bU \in \mathbb{R}^{k \times (k-1)}$ are an orthonormal basis for $\Delta_k$.
Then,
\begin{align} \label{eqn:quad} 
\lambda \bB \bp^{-1} = \bq - \bp
\end{align}
for some $\lambda \in \mathbb{R}$
is a necessary condition for any $\bp$ that satisfies 
$$\argmin_{\bp \in \Delta_k: f_{\ell}(\bp) = 0} ||\bp-\bq||_2.$$

\begin{proof}
Recall $f_\ell(\bp) = 1 - c_0 p_1^{c_1} p_2^{c_2}  \cdots p_k^{c_k}$.   The closest point $\{\bp \in \Delta_k: f_\ell(\bp)= 0 \}$ and a fixed point $\bq \in \Delta_k$
must satisfy the \emph{orthogonality condition} \cite{gubner2010ece}, Thm. 3.11, which  implies that the shortest distance between a point and a surface be normal to the surface.  Note that the normal vector of the discontinuity variety $\{\bp \in \Delta_k: f_\ell(\bp)=\gamma \}$ is the gradient of its level set  function $f_\ell(\bp)$ projected onto the simplex. By the orthogonality condition, if the columns of $\bU$ are a basis for $\{\bx \in \mathbb{R}^k : \bx^T \boldsymbol{1} =0\}$, we require
\begin{align} \label{eqn:orth1}
    \lambda \bU \bU^T \nabla f_\ell(\bp)= \bq-\bp 
\end{align}
of any point that minimizes the distance to $\bq$, where the $i$th element of $\nabla f_\ell(\bp)$ is
\begin{align}
\left[\nabla f_\ell(\bp)\right]_i = c_i p_i^{-1} (f_\ell(\bp)-1).
\end{align}
Together with (\ref{eqn:orth1}), this implies
\begin{align}
\lambda (f_{\ell}(\bp) -1) \bU \bU^T \bC \bp^{-1} = \bq - \bp.
\end{align}
where $\bp^{-1} = [p_1^{-1} \  \dots  \ p_k^{-1}]^T$ and $\bC = \mathrm{diag}(c_1, \dots, c_k)$.  Since $(f_\ell(\bp)-1) \in \mathbb{R}$, this gives the result.  We provide further details in Appendix \ref{app:orth_cond}.
\end{proof}
\end{cor}

Our goal is to find \emph{all} points that satisfy the necessary condition in Cor. \ref{cor:orth} for $f_\ell(\bp) = 0$.  Note that (\ref{eqn:quad}) is a system of polynomial equations, and each equation involves a quadratic term of a single variable.  In some settings, this can be solved explicitly (see \cite{cox2013ideals}, chapter 3) using Grobner bases. 

\begin{obs}
Fix $\lambda$.  For $k=3$ the solutions to (\ref{eqn:quad}) can be found by a numerical procedure equivalent to finding the roots of a single variable differentiable function over an interval. See App. \ref{sec:quad_roots}. 
\end{obs}

\begin{obs}
Let $\bp_{j}(\lambda)$ enumerate the solutions to (\ref{eqn:quad}). Then, $f_{\ell}(\bp_{j}(\lambda)) : \mathbb{R} \mapsto \mathbb{R}$ is a differentiable function. Finding solutions to $f_{\ell}(\bp_{j}(\lambda)) = 0$ is equivalent to finding a roots of a single variable differentiable function.  See App. \ref{sec:grad}.
\end{obs}

Equipped with the ability to compute the distance to an individual variety, we next compute the minimum distance from an arbitrary point to a \emph{continuity set}. Consider only the case for $k=3$.  Suppose a continuity set contains discontinuity varieties $f_{i_1}, \dots\, f_{i_m}$, and vertices $\bv_1, \dots\, \bv_l$.   We can compute the distance from a point $\bq$ to $f_{i_1}, \dots\, f_{i_m},\bv_1,\dots\,\bv_l$.  Notice that distance from an arbitrary point to a discontinuity variety may not be feasible for that continuity set (which is readily verified). This point is then excluded computing the minimum.

To complete this process, we can evenly grid the whole simplex so that grid points are separated by at most $\epsilon$. Next, we only pick points inside of continuity set, or whose distance to the continuity set are less than or equal to $\delta$. Then these points form a $(\epsilon, \delta)$-cover for the continuity set.  The approach is detailed in Alg. \ref{alg:findd}.

\begin{cor}
$\mathcal{G}_{\o}$ is an $(\epsilon, \delta)$-cover of $\Reg_{\o}$.
\end{cor}
\begin{proof}
Note that $\Delta_{k, \eta}$ is an $(\epsilon, 0)$ cover of $\Delta_k$ when $\eta = \left \lceil \frac{\sqrt{k}}{\epsilon} \right \rceil$ (see 
Appendix \ref{app:eps_eta}).
Since $\delta \geq \epsilon$, none of the points that cover any portions of the continuity set are discarded.  Every point that remains in  $\mathcal{G}_{\o}$ is closer than $\delta$ to at least one point in $\Reg_{\o}$, and we conclude the result.  
\end{proof}

\begin{figure}
\centering
\includegraphics[width=0.44\textwidth]{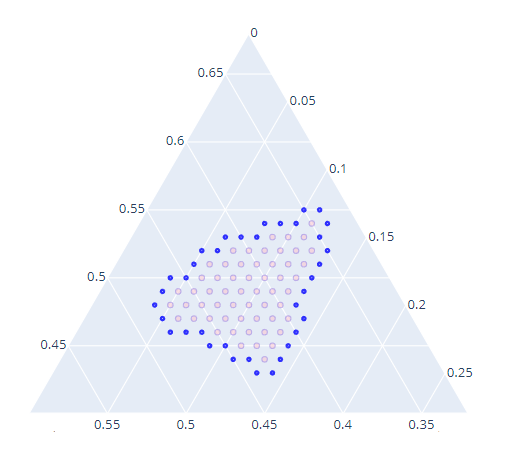}
\caption{An example of $(\epsilon, \delta)$-cover of $\Reg_{\o}$. Grid points are $\delta$ away from each other. Blue dots are outer points with distance less than $\epsilon$ away from $\Reg_{\o}$. Pink dots are inside $\Reg_{\o}$. Figure generated by Plotly \cite{plotly}.
\label{fig:delta_epsilon_cover}}
\end{figure}

\begin{algorithm} 
\caption{ \label{alg:findd} Finding an $(\epsilon, \delta)$-cover for $\Reg_{\o}$ }
\begin{algorithmic}[1] 
\State{\textbf{Input}}: $\delta \geq \epsilon \geq 0$, $\phat  \in \Delta_{k, n}$, continuity set $\Reg_{\o}$ with vertices $\mathcal{V} = \{v_1,...v_l\}$ and discontiuity varieties $\mathcal{F} = \{f_{i_1},\dots\, f_{i_m}\}$ that touch $\Reg_{\o}$.
\State {initialize: set $\mathcal{G}_{\o}=\{ \}$ and 
$\eta = \left \lceil \frac{\sqrt{k}}{\epsilon} \right \rceil$} 
\For {$p \in \Delta_{k,\eta}$}
\State $\mathcal{Q}$ = $ \mathcal{V} \cup \{q_j: \argmin_{q_j:j = 1, \cdots m} d(p,q_j), p \in f_{i_j} \cap \Reg_{\o}, f_{i_j} \in \mathcal{F} \} $
\If {$p \in \Reg_{\o}$ or $\min d(p, q)\leq \delta, q \in \mathcal{Q} $}
\State {Append $p$ to $\mathcal{G}_{\o}$}
\EndIf
\EndFor

\State {\textbf{Return}: $(\epsilon, \delta)$-cover $\mathcal{G}_{\o}$}

\end{algorithmic}
\end{algorithm}

\section{Applications: A/B Testing}
Suppose we have two empirical observations, $\widehat \bp_1$ and $\widehat \bp_2$, each which have a MVCs, $\mathcal{C}^{\star}_{\delta}(\widehat \bp_1)$ and $\mathcal{C}^{\star}_{\delta}(\widehat \bp_2)$.  Our aim is to determine whether the two MVCs intersect.

Given the full enumeration of the continuity sets and their corresponding $(\epsilon, \delta)$ covers, the $p$-value function can be fully specified numerically, and the corresponding level sets can be computed to required precision.  More precisely, for determining an empty intersection, we consider the following approach.  Recall the splitting inequalities will split each confidence sets into continuity sets, where $p$-value is continuous under each continuity sets. Because we are now measuring two confidence sets, we want to make sure that in each continuity set, the $p$-value functions for both confidence sets are continuous. We can achieve this by including the splitting inequalities for both confidence sets and computing the corresponding continuity sets. This will double the number of splitting inequalities.

To continue, we cover the new continuity sets, and measure the $p$-value function under $\widehat{\bp}_1$ and $\widehat{\bp}_2$ for each point in $\mathcal{G}_\o$. We can bound the $p$-value over the ball of radius $\epsilon$, denoted $R_\epsilon(\bp)$, using the Lipschitz constant $L$:
\begin{align} \label{eqn:lip1}
    \left \vert \rho_{\phat_i}(\bq) - \rho_{\phat_i}(\bp) \right \vert\leq L \left \vert\bq - \bp\right \vert\leq L\epsilon, \bq \in R_\epsilon(\bp)
\end{align}
which implies $
    \rho_{\phat_i}(\bp)- L\epsilon \leq \rho_{\phat_i}(\bq) \leq \rho_{\phat_i}(\bp)+ L\epsilon$.
See App. \ref{sec:grad}) to bound the Lipschitz constant $L$.

However, it is possible that parts of the ball are in both confidence sets, which means $\rho_{\phat}(\bp)+ L\epsilon > 1 - \alpha$ for the $p$-value of a neighboring point in both sets, and it will remain unclear whether their intersection is empty. In this case, we decrease the value of $\epsilon$ and repeat the approach until such situation no longer happens. 

\section{Summary}
In this paper we studied the {\em minimum volume confidence sets} for the multinomial parameter.  We showed that the MVCs, which are prescribed as the level sets of a discontinuous function (the exact $p$-value) can be computed to arbitrary precision by enumerating and covering the regions over which they are continuous.  This allowed us to answer a basic question in A/B testing in the restricted setting of three categories.  While the approach sheds light on the general setting of more than three categories, defining a covering for the continuity sets in general remains an open problem.  The primary challenge lies in covering the lower dimensional faces and edges of the continouity sets.

\bibliography{main}
\bibliographystyle{IEEEtran}

\onecolumn

\section*{Appendix}

\subsection{Continuity Regions}
\label{app:verts}
\begin{cor} \label{cor:vertex2} 
For $k=3$, the number continuity regions whose closure contains zero vertices (termed a zero-vertex continuity region) is less than or equal to $m$.
\begin{proof}
When $k=3$, a zero-vertex continuity region has at most one active splitting equality (a continuity region has an active splitting equality if its closure contains a point such that $f_{\ell}(\bp) = 0 $ for some $\ell$), or it would have one or more vertices.  Note this is not necessarily true for $k>3$, since a zero-vertex continuity region could have more than one active splitting equality. 

Moreover, a splitting equality can be active for at most one zero-vertex continuity region. To see this note $\{\bz \in \tilde{\Delta}_k : \bz^T \bc > c \}$ is a connected set for $k=3$.  There are $m-1$ splitting equalities; this accounts for at most $m-1$ zero-vertex continuity regions. 

The possibility of a zero vertex set with no active splitting equalities remains. This region would contain the entire simplex; hence there can be at most one. 
\end{proof}
\end{cor}

\begin{cor} \label{cor:vertex3}
For $k=3$, there are at fewer than $8(n+2)^4$ continuity regions.  
\begin{proof}
The proof follows from Cor. \ref{cor:vertex1} and the observation that each vertex can belong to at most 4 continuity regions, which implies there are at most $8 { m-1 \choose k-1}$ continuity regions with one or more vertices.  Since there at most $m$ zero-vertex continuity regions from Cor. \ref{cor:vertex1}, we conclude
\begin{eqnarray*}
r &\leq& 8  {m-1 \choose k-1} + m \leq 8 m^2 \leq 8 (n+2)^4.
\end{eqnarray*}
\end{proof} 
\end{cor}

\subsection{Proof of $\eta$, $\epsilon$ relationship} \label{app:eps_eta}

Given an arbitrary point in the continuous simplex, $\bp \in \Delta_k$, consider finding a \emph{discrete neighbor} of $\bp$, denoted $\bp'$, where $\bp' \in \Delta_{k,\eta}$ is given by the following procedure.  First set $p_1' = \mbox{round}(\eta p_1)/\eta$.  If $p_1 - p_1' < 0$, set $p_2' = \lceil p_2 \eta \rceil /\eta$, otherwise $p_2' = \lfloor p_2 \eta \rfloor /\eta$.  If $\sum_{j=1,2} (p_j-p_j') < 0$, set $p_3' = \lceil p_3 \eta \rceil /\eta$, and $p_3' = \lfloor p_3 \eta \rfloor /\eta$ otherwise.  Continuing until the $k$th entry of $\bp'$, and set $p_k' = 1 - \sum_{i<k} p_i'$.   By construction $\bp' \in \Delta_{k,\eta}$ and we have $|p_i - p_i' | < 1/\eta$ for $i = 1,\dots, k$ which is trivial for $i=1,\dots,k-1$.  For the last entry, note that we have $\sum_{i < k} (p_i-p_i') \leq 1/\eta$  since each term deviates by at most $1/\eta$, and we always choose its sign to minimize the absolute value of the quantity. 

Since $|p_i - p_i' | < 1/\eta$, we conclude 
\begin{align}
    \left\vert \left \vert \bp - \bp' \right\vert \right\vert \leq \sqrt{\frac{k}{\eta^2}}.
\end{align}
Setting $\eta \geq \frac{\sqrt{k}}{\epsilon}$ ensures that we cover any point $\bp$ in the continuous simplex with a discrete point that is at most $\epsilon$ away. Thus, if $\eta \geq \frac{\sqrt{k}}{\epsilon}$, then $\Delta_{k,\eta}$ is an $(\epsilon, 0)$ cover of $\Delta_k$.

\subsection{Orthogonality Condition} \label{app:orth_cond}
We provide further details on the necessary conditions for a minimum distance point on a variety in the simplex. 
By definition, the gradient $\nabla f_{\ell}(\bp)$ is perpendicular to the tangent plane of the $k-1$ dimensional variety $\{\bp : f_{\ell}(\bp) = 0 \}, \bp \in \mathbb{R}^k$.  Define the $(k-1)$-d subspace that spans the tangent plane, namely $\bV = \{\bx \in \mathbb{R}^k: \bx^T \nabla f_{\ell}(\bp) = 0\}$.  Likewise, consider the $(k-1)$-d subspace $\bU = \{\bx \in \mathbb{R}^k: \bx^T \boldsymbol{1} = 0\}$.  Define the $(k-2)$-d subspace 
\begin{eqnarray}
\bV_{\Delta} = \bV \cap \bU = \left\{\bx : \begin{bmatrix} \boldsymbol{1}^T \\ \nabla f_{\ell}(\bp)^T \end{bmatrix} \bx = \begin{bmatrix} 0 \\ 0 \end{bmatrix} \right\}
\end{eqnarray}
which represents the tangent plane of the $k-2$ dimensional variety \emph{in the simplex}.   Next, to establish the orthogonality conditions in this setting, first note that $\mathrm{proj}_{\bU} \nabla f_{\ell}(\bp)$ is perpendicular to any $\bx \in \bV_{\Delta}$. To see this, let the columns of $U$ be an orthonormal basis for the subspace $\bU$.  Then $\mathrm{proj}_{\bU} \nabla f_{\ell}(\bp) = U U^T  \nabla f_{\ell}(\bp)$.  For any $\bx \in \bV_\Delta$,  $(U U^T  \nabla f_{\ell}(\bp))^T \bx =   \nabla f_{\ell}(\bp)^T U U^T  \bx = 0$ since the columns of $U$ are orthogonal to $\bx \in \bV_\Delta$ by definition.  

Consider a point $\bq \not \in S \subset \Delta_k$. The orthogonality condition \cite{gubner2010ece} requires that if a point $\bp^* \in S \subset \Delta_k$ satisfies 
\begin{eqnarray}
\bp^* =  \arg\min_{\bp \in S} ||\bq - \bp||^2
\end{eqnarray}
then $\bp^*-\bq$ must be orthogonal to the tangent plane of $S$ at $\bp^*$.  The orthogonality condition, together with the fact the the tangent plane of the variety $\{\bp \in \Delta_k : f_{\ell}(\bp) = 0 \}$ is orthogonal to $U U^T  \nabla f_{\ell}(\bp)$ implies the necessary condition shown in Cor. \ref{cor:orth}.



\subsection{Solutions to Quadratic System} \label{sec:quad_roots}
Note that we can enumerate and find the solutions to the quadratic system of equations in the specific case of $k=3$.  In particular, $\lambda \bB \bp^{-1} = \bq - \bp$, for fixed $\lambda$, is equivalent to 
\begin{align*}
   a_{1,1} x + a_{2,1} y + a_{3,1} z + d_1 = \nicefrac{1}{x} \qquad (1) \\
    a_{2,1}  x + a_{2,2} y + a_{3,2} z + d_2  = \nicefrac{1}{y} \qquad (2) \\
    a_{3,1}  x + a_{3,2} y + a_{3,3} z + d_3 = \nicefrac{1}{z} \qquad (3) 
\end{align*}
for variables $x, y, z \in \mathbb{R}$ and constants $a_{i,j}$ (note that $\bB$ is symmetric by definition).   
Solving (1) and (3) for $y$ and equating gives
\begin{align*}
\nicefrac{1}{a_{2,1}} \left(\nicefrac{1}{x} - a_{1,1} x - a_{3,1} z - d_1 \right) = \nicefrac{1}{a_{2,3}} \left(\nicefrac{1}{z} - a_{3,1} x - a_{3,3} z - d_3 \right)
\end{align*}
which after multipliyng both sizes by $xz$, is quadratic in both $x$ and $z$:
\begin{align*}
(a_{2,1} a_{3,1}  - a_{2,3} a_{1,1}) x^2 z +
(a_{2,1} a_{3,3}  - a_{2,3} a_{3,1}) xz^2 +
(a_{2,1} d_3  -  a_{2,3} d_1) xz +
a_{2,3} z 
- a_{2,1}x = 0
\end{align*}
We can write the roots of $x$ as an explicit function of $z$:
\begin{align*}
(x_z^{+}(z), x_z^{-}(z)) = \frac{-b \pm \sqrt{b^2 - 4ac}}{2a}
\end{align*}
which, for simplicity of notation, we denote as $x_z^{+}(z)$ and $x_z^{-}(z)$, where
\begin{eqnarray*}
a &=& (a_{2,1} a_{3,1}  - a_{2,3} a_{1,1}) z \\
b &=& (a_{2,1} a_{3,3}  - a_{2,3} a_{3,1}) z^2 + (a_{2,1} d_3  -  a_{2,3} d_1)z - a_{2,1} \\
c &=& a_{2,3} z.
\end{eqnarray*}
Following the same procedure, we can solve (2) and (3) for $x$ and equate the two equations, writing the two roots of $y$ as a function for $z$; namely, $y_z^{+}(z)$ and $y_{z}^{-}(z)$.  Lastly, solving (1) and (2) for $z$, and equating, we can write the roots of $x$ as a function of $y$, denoted $x_y^{+}(y)$ and $x_y^{-}(y)$.

Equipped with explicit expressions for $x$ as a function of $y$ and for $x$ as a function $z$, and an explicit relationship between $y$ and $z$, we have eight parametric expressions of the form involving only a single variable, $z$:
\begin{eqnarray*}
    x_z^{+}(z) &\overset{(1)}{=}& x_y^{+}\left(y_z^{+}(z) \right) \\
    x_z^{+}(z) &\overset{(2)}{=}& x_y^{+}\left(y_z^{-}(z) \right) \\
    \vdots \\
    x_z^{-}(z) &\overset{(8)}{=}& x_y^{-}\left(y_z^{-}(z) \right) 
\end{eqnarray*}
where we consider equating all combinations of roots involving combinations $\pm$.  As the roots are differentiable functions of their arguments, we can find and enumerate the solutions to each of the single variable equations numerically.

\subsection{Bounds on Gradient} \label{sec:grad}
\noindent
Here we show how to upper bound the gradient for $p$-value function for a fixed outcome $\widehat{\bp}$,
\begin{align*}
\left \vert\nabla \rho_{\phat}(\bp)\right \vert &= \left \vert\nabla \sum_{\qhat \in \Delta_{k,n} : \mathbb{P}_{\bp}(\qhat) \leq \mathbb{P}_{\bp}(\phat)  } \mathbb{P}_{\bp}  \left( \qhat  \right)\right \vert
\leq \left \vert \nabla \sum_{\qhat \in \Delta_{k,n}} \mathbb{P}_{\bp}  \left( \qhat  \right) \right \vert\\
&=\left \vert\begin{bmatrix}
      \sum_{\qhat }\frac{n!}{(n\widehat{p}_1-1)!(n\widehat{p}_2)!  \ldots (n\widehat{p}_k)!} p_1^{n\widehat{p}_1-1}p_2^{n\widehat{p}_2}  \cdots p_k^{n\widehat{p}_k}\\
       \sum_{\qhat }\frac{n!}{(n\widehat{p}_1)!(n\widehat{p}_2-1)!  \ldots (n\widehat{p}_k)!} p_1^{n\widehat{p}_1}p_2^{n\widehat{p}_2-1}  \cdots p_k^{n\widehat{p}_k}\\
       \vdots \\
       \sum_{\qhat }\frac{n!}{(n\widehat{p}_1)!(n\widehat{p}_2)!  \ldots (n\widehat{p}_k-1)!} p_1^{n\widehat{p}_1}p_2^{n\widehat{p}_2}  \cdots p_k^{n\widehat{p}_k-1}
     \end{bmatrix}\right \vert
     \leq \left \vert\begin{bmatrix}
      \sum_{\qhat }\frac{n!}{(n\widehat{p}_1-1)!(n\widehat{p}_2)!  \ldots (n\widehat{p}_k)!}\\
       \sum_{\qhat }\frac{n!}{(n\widehat{p}_1)!(n\widehat{p}_2-1)!  \ldots (n\widehat{p}_k)!} \\
       \vdots \\
       \sum_{\qhat }\frac{n!}{(n\widehat{p}_1)!(n\widehat{p}_2)!  \ldots (n\widehat{p}_k-1)!}
     \end{bmatrix}\right \vert = L
\end{align*}
where first inequality is followed by the fact that any partial derivatives of multinomial probability  $\mathbb{P}_{\bp}  \left( \qhat  \right)$ is positive. The second inequality is followed by the fact that $p_i \in [0,1], i = 1,2,\dots k$.

\end{document}